\documentclass[onefignum,onetabnum]{siamonline190516}
\PassOptionsToPackage{numbers}{natbib}
\usepackage[utf8]{inputenc}
\usepackage{amsfonts,amssymb,amsbsy}
\usepackage{graphicx}
\usepackage{booktabs}       % professional-quality tables
\usepackage{float}
\usepackage{natbib}
\usepackage{xspace}
\usepackage{hyperref}
\usepackage[mathscr]{euscript}
\usepackage{bm}
\usepackage{dsfont}
\usepackage{algorithm}
\usepackage{algorithmic}
\usepackage[]{epstopdf}
\newtheorem*{remark}{Remark}

\newcommand{\R}{\ensuremath{\mathbb{R}}}
\renewcommand{\vec}[1]{\ensuremath{\mathbf{#1}}}
\newcommand{\mat}[1]{\ensuremath{\mathbf{#1}}}
\newcommand{\Xten}{\ensuremath{\bm{\mathscr{X}}}}
\newcommand{\Aten}{\ensuremath{\bm{\mathscr{A}}}}
\newcommand{\Yten}{\ensuremath{\bm{\mathscr{Y}}}}
\newcommand{\ten}[1]{\ensuremath{\bm{\mathscr{#1}}}}
\newcommand{\Xmat}{\ensuremath{\mathbf{X}}}
\newcommand{\Amat}{\ensuremath{\mathbf{A}}}
\newcommand{\Ymat}{\ensuremath{\mathbf{Y}}}
\newcommand{\uvec}[1]{\ensuremath{\mathbf{u}^{(#1)}}}
\newcommand{\U}[1]{\ensuremath{ {\mathbf{U}^{(#1)}}} }
\newcommand{\Dk}{\ensuremath{\mathbf{D}^{(k)}}}

\ifpdf
\hypersetup{
  pdftitle={Time-varying Autoregression with Low Rank Tensors},
  pdfauthor={KD Harris, A Aravkin, R Rao, B Brunton}
}
\fi

\begin{document}

\title{
  Time-varying Autoregression with Low Rank Tensors
}
\author{
  Kameron Decker Harris\thanks{Corresponding author: 
    Computer Science \& Engineering, Biology, University of Washington, \url{kamdh@uw.edu}}
  \and
  Aleksandr Aravkin\thanks{Applied Mathematics, University of Washington, \url{saravkin@uw.edu}}
  \and
  Rajesh Rao\thanks{Computer Science \& Engineering, University of Washington, \url{rao@cs.washington.edu}}
  \and
  Bingni Wen Brunton\thanks{Biology, University of Washington, \url{bbrunton@uw.edu}}
}
% \author{
%   Kameron Decker Harris
% }
%   % \thanks{University of Washington, Computer Science \& Engineering;
%   %   \texttt{kamdh@uw.edu}}
%   % \and
% \author{Aleksandr Aravkin}
% \author{Rajesh Rao}
% \author{Bingni Wen Brunton}
\date{\today}
\maketitle

\begin{abstract}
  We present a windowed technique to learn 
  parsimonious time-varying autoregressive models from 
  multivariate timeseries.
  This unsupervised method uncovers interpretable spatiotemporal structure in data
  via non-smooth and non-convex optimization.
  In each time window, we assume
  the data follow a linear model parameterized by a system matrix,
  and we model this stack of potentially different system matrices as a low rank tensor.
  Because of its structure, the model is scalable to high-dimensional data
  and can easily incorporate priors such as smoothness over time.
  We find the components of the tensor using alternating minimization
  and prove that any stationary point of this algorithm is a local minimum.
  We demonstrate on a synthetic example that our method identifies the true rank of a switching linear
  system in the presence of noise. 
  We illustrate our model's utility and superior scalability over 
  extant methods when applied to 
  several synthetic and real-world example:
  two types of time-varying linear systems, worm behavior, sea surface temperature,
  and monkey brain datasets.
\end{abstract}

\section{Introduction}

Data-driven linear models are a common approach to modeling multivariate timeseries
and have been studied extensively in the mathematical sciences and applied fields.
These domains include
Earth and atmospheric sciences \citep{leith1978a,danforth2007,alexander2008},
fluid dynamics \citep{rowley2009,allgaier2012,budisic2012,towne2018},
and neuroscience \citep{brunton2016,markowitz2018}.
Even when the underlying dynamics are nonlinear,
such linear approximations may be justified by their
connection to the eigenfunctions of the Koopman operator
\citep{budisic2012,lusch2018},
known as the dynamic mode decomposition
\citep{rowley2009,tu2014,takeishi2017}.
% Linear models are a classical, statistically justified method
% for modeling stationary timeseries, 
% but in many systems the true dynamics are nonlinear or vary over time.
In order to capture non-stationary behavior,
it is important to leverage
dynamic linear models (DLMs)
that vary over time \citep{west1997a}.
However, 
since every linear model is parametrized by a 
{\em system matrix},
and this matrix changes over time,
a naive DLM fit to a length $T$ timeseries of $N$ variables 
has $\mathcal{O}(T N^2)$ parameters, which could be very large.
We propose to manage such complexity by representing the 
stack of system matrices as a low rank tensor 
with only $\mathcal{O}(T + 2N)$ parameters
(Figure~\ref{fig:schematic}).

There is a rich and extensive literature on DLMs,
including
time-varying autoregressive (TVAR) and
switching linear dynamical systems (SLDS) models
that we cannot review in full here \citep{west1997a}.
Ours is a regularized optimization method
\citep[in the spirit of][]{ghahramani2000,chan2014,chan2015a,yu2016b,tank2017a},
complementary to the Bayesian approaches taken
in much of the literature
\citep{prado2000,fox2009,damien2013,linderman2017}.
We also would like to highlight the recent work of 
\cite{costa2019}, which uses likelihood tests to 
adaptively segment a timeseries and fit different models
to each segment.
However, our approach is inherently more scalable due to 
the low rank assumptions we make.

The key innovation of our model is to parametrize the dynamics
by a low rank tensor for
computational tractability, ease of identification, and
interpretation.
Tensor decompositions \citep{kolda2009} are a powerful technique
for summarizing multivariate data
and an area of ongoing research in theory
\citep[][among others]{ge2017,anari2018}
and applications
\citep[e.g., to improve neural networks][]{novikov2015,he2017}.
In our formulation, the system tensor representing
the DLM is regressed against the data.
In this aspect, our method is most similar to
the work of Yu and colleagues
who considered spatiotemporal forecasting
with spatial smoothness regularization
\citep{bahadori2014,yu2016,yu2018},
in contrast to our temporal smoothness,
although the possibility is mentioned in their dissertation \citep{yu2017a}.
Our work also differs in the emphasis on non-smooth regularization
to find switching or other temporally structured behavior.

\begin{figure*}[]
  \centering
  \includegraphics[width=\textwidth]{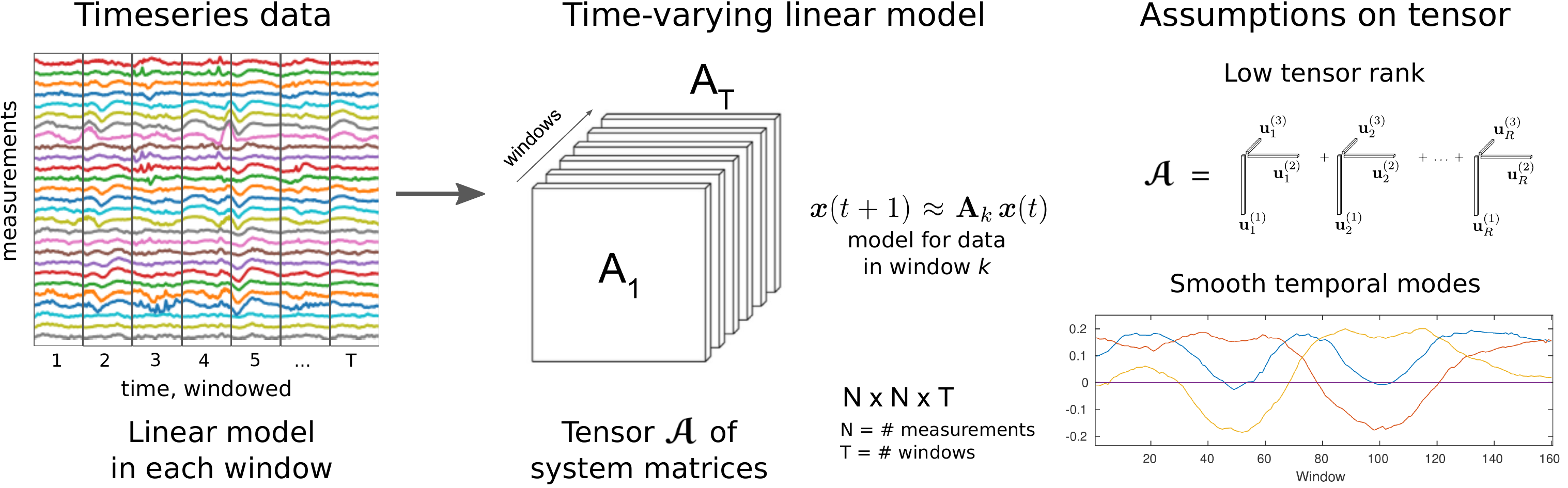}
  \caption{
    Schematic of the TVART method.
    Time series data are windowed, and a linear model is fit in each window.
    The system matrices across windows are assumed to 
    arise from a tensor that is low rank.
    This assumption couples together the different matrices in time
    and factors as a spatiotemporal modal decomposition: 
    The left and right spatial 
    modes $\U1$ and $\U2$ are a basis
    for the row and column spaces of the system matrices, 
    and the temporal modes
    $\U3$ allow their weights to change over time.
    The temporal modes shown in the bottom right correspond to the columns
    of $\U3$.
    We allow the possibility of 
    smoothing of the temporal modes in 
    order to stabilize the model fit with small window size.
  }
  \label{fig:schematic}
\end{figure*}

\subsection*{Notation}
\label{sec:notation}

We follow the conventions of
\cite{kolda2009} for notation.
In brief, tensors are denoted with calligraphic bold ($\ten{A}$), 
matrices with capital bold ($\mat{A}$), and vectors with lower-case bold symbols
($\bf a$).
MATLAB-style colon notation represents slices,
e.g.\ ${\bf A}_{i:}$ is the row vector formed from the
$i$th row of a matrix ${\bf A}$.
All of the tensors we consider are third-order.
For a tensor $\ten{A}$, let 
${\bf A}_k = {\ten A}_{::k}$ be its $k$th frontal slice.

\section{The TVART model}

\label{sec:problem_setup}

We now introduce our time-varying autoregressive model with low rank tensors 
(TVART, Figure~\ref{fig:schematic}).
Assume we have sampled the trajectory of a 
dynamical system
$\bm{x}(t) \in \R^N$ for $t = 1, \ldots, \tau+1$.
We split this trajectory into $T$ 
non-overlapping windows of length $M$,
so that $T M = \tau$.
Let $\Xten$ be the $N \times M \times T$ tensor with entries
$
(\Xten)_{ijk} = x_i ( (k-1)M + j ),
$
and similarly let $\Yten$ be a tensor of the same size with entries shifted 
by one time point,
$
(\Yten)_{ijk} = x_i ( (k-1)M + j + 1 )
$.
We call the frontal slices 
$\Xmat_k, \Ymat_k \in \R^{N \times R}$ 
the {\em snapshot matrices} for window $k$.
The first of these are
%\Xmat_k = \left[ \bm{x}( (k-1) M + 1) , \bm{x}( (k-1) M + 2), \ldots \bm{x}( k M) \right]
\[
\Xmat_1 = \left[ 
  \begin{array}{cccc}
    \vert & \vert & & \vert \\
    \bm{x}(1) &\bm{x}(2) &\ldots &\bm{x}(M) \\
    \vert & \vert & & \vert 
  \end{array}
\right]
%\]
%and
%\[
\quad \mbox{and} \quad
\Ymat_1 = \left[ 
  \begin{array}{cccc}
    \vert & \vert & & \vert \\
    \bm{x}(2) &\bm{x}(3) &\ldots &\bm{x}(M+1) \\
    \vert & \vert & & \vert 
  \end{array}
\right] .
\]
The subsequent snapshots $\Xmat_k, \Ymat_k$ for $k > 1$
are each shifted by $(k-1)M$.

The goal of TVART is to fit an 
$N \times N \times T$
tensor $\Aten$ of system matrices, so that
$
\Ymat_k \approx \Amat_k \Xmat_k 
$
for $k = 1, \ldots, T$,
where $\Amat_k$ is the $k$th frontal slice of $\Aten$. 
The assumption underlying this goal is that 
$
\bm{x}(t+1) \approx \Amat(t) \, \bm{x}(t)
$
where $\Amat(t)$ is constant within a window.
This assumption motivates the least squares optimization problem below,
which is equivalent to assuming uncorrelated Gaussian errors:
\begin{equation}
  \label{eq:TVART}
  \min_{\Aten: \; \mathrm{rank} (\Aten) = R} 
  %L(\Aten)
  %=
  %\min_{\Aten: \; \mathrm{rank} (\Aten) = R} 
  \, 
  \frac{1}{2}\sum_{k=1}^T \| \Ymat_k - \Amat_k \Xmat_k \|_F^2 \, .
  \tag{TVART}
\end{equation}
% \begin{equation}
% \label{eq:lsq_windowed_DMD}
% \min_{\Aten} \frac{1}{2} \sum_{k=1}^T \| \Ymat_k - \Amat_k \Xmat_k \|_F^2 \, .
% \end{equation}
% This can be more succinctly denoted with the face-wise 
% tensor product $\vartriangle$ and tensor (Frobenius) norm:
% \begin{equation*}
% \min_{\Aten} \frac{1}{2} \| \Yten - \Aten \vartriangle \Xten \|^2 \, .
% \end{equation*}
Without the rank constraint, \ref{eq:TVART}
factors into decoupled problems for each window $\Amat_k$.
In order to limit the degrees of freedom in the tensor $\Aten$,
we use a low rank formulation.
Specifically, we represent $\Aten$ using the
canonical polyadic (CP) decomposition \citep{kolda2009} of rank $R$  as 
\[
\Aten = 
\sum_{r=1}^R %\lambda_r \, 
\uvec1_r \circ \uvec2_r \circ \uvec3_r
\]
in terms of the factor matrices 
$\U1 \in \R^{N \times R}$, 
$\U2 \in \R^{N \times R}$,
and 
$\U3 \in \R^{T \times R}$,
where 
%$\uvec{i}_r$ is the $r$th column of $\U{i}$:
\begin{equation*}
  \U{i} = 
  \left[
    \begin{array}{ccc}
      \vert & & \vert \\
      \uvec{i}_1 & \ldots & \uvec{i}_R \\
      \vert & & \vert 
  \end{array}
  \right] .
\end{equation*}
Thus, the number of parameters is reduced to $(2N + T)R$,
which is now linear in $N$ and $T$.
We can optionally normalize the factors to have unit-length columns
and capture their scalings in a vector 
$\bm{\lambda} \in \R^R$;
we consider this a postprocessing step and explicitly state when we do this.
% \begin{equation}
%   \min_{\Aten : \; \mathrm{rank} (\Aten) = R} \; \;
% \frac{1}{2}   \| \Yten - \Aten \vartriangle \Xten \|^2 \, .
%   \tag{TVART}
% \end{equation}

When solving \ref{eq:TVART} for $\Aten$, 
we work directly with its frontal slices
\begin{equation}
  \label{eq:slice_modes}
  \Amat_k 
  = \U1 \Dk {\U2}^\intercal ,
  \quad \mbox{where} \quad
  \Dk = \mathrm{diag} \left( \uvec3_{k:} \right).
\end{equation}
The matrix $\Dk$ is the $R \times R$ diagonal matrix formed from the $k$th row of $\U3$ \citep{kolda2009}.

\begin{definition}
We call the matrices $\U1$, $\U2$, and $\U3$ the
{\em TVART dynamical modes}.
Specifically, we refer to $\U1$ and $\U2$ 
as the {\em left and right spatial modes},
since they determine the loadings of $\Amat_k$ 
onto the spatial dimensions/channels in the data.
The matrix $\U3$ contains the {\em temporal modes}, 
since it determines the time-variation of the 
system matrix $\Amat_k$ across windows.
\end{definition}

In some ways, \eqref{eq:slice_modes} is similar to the
singular value decomposition (SVD): 
each slice $\Amat_k$ is the product of a 
low rank matrix $\U1$, a diagonal matrix $\Dk$,
and another low rank matrix $\U2^\intercal$.
However, unlike the SVD, 
the left and right spatial modes are not orthogonal.
Let 
$\U1 = \mat{Q}^{(1)} \mat{R}^{(1)}$ and
$\U2 = \mat{Q}^{(2)} \mat{R}^{(2)}$
be the QR decompositions of the left and right spatial modes,
so that
$\Amat_k = \mat{Q}^{(1)} \mat{R}^{(1)} \Dk \mat{R}^{(2) \intercal} \mat{Q}^{(2) \intercal}$.
Thus, in order to calculate the SVD of $\Amat_k$, we would have to take the SVD of the 
$R \times R$ matrix $\mat{R}^{(1)} \Dk \mat{R}^{(2) \intercal}$.
In the CP decomposition, slices $\mat{A}_k$
and $\Amat_{k'}$ may have different left and right singular vectors,
but these singular vectors are always in the span of
$\mat{Q}^{(1)}$ and $\mat{Q}^{(2)}$.
This flexibility allows \ref{eq:TVART} 
to fit multiple linear models with different singular subspaces.

\subsection{Extensions: affine dynamics and higher-order autoregressions}

In many applications, 
the mean of the data may drift over time, 
and thus affine models of the dynamics
$\vec{x}(t+1) = \mat{A}_k \vec{x}(t) + \vec{b}_k$
are more appropriate than linear models
\citep{hirsh2019}.
We can fit an affine model of this type
within the TVART framework by appending a row
of ones to each $\mat{X}_k$
and extending $\U2$ by one row to build in a $\vec{b}_k$ term.
In this case, we have that 
$\vec{b}_k = \U1 \Dk \vec{c}$, 
where $\vec{c}$ is the extra row $\uvec2_{N+1,:}$.

Furthermore, autoregressive models of higher order are often considered,
where $\vec{x}(t+1)$ is predicted from data with $P$ lags
$\vec{x}(t), \ldots, \vec{x}(t - P + 1)$ \citep{west1997a,prado2000}.
In this case, the dimensions of $\Xten$, $\Aten$, and $\U2$ change
to $NP \times M \times T$, $N \times NP \times T$, and $NP \times R$,
respectively, but otherwise the mathematics remain equivalent.
For simplicity, we focus on just the $P = 1$ case,
but higher-order autoregressive models
are likely better-suited to certain applications.

\section{Alternating least squares algorithm}

\label{app:ALS}

We now describe in detail the optimization routine
used to solve \eqref{eq:TVART}.
% This is an unregularized problem,
% but the following is also needed for solving
% the eventual regularized problem we propose
% in Section~\ref{sec:regularization}.
Define the loss function to be
\begin{equation}
\label{eq:loss}
L \left( \U1, \U2, \U3 \right)  = 
\frac{1}{2} \sum_{k=1}^T 
\| \Ymat_k - \U1 \Dk {\U2}^\intercal \Xmat_k \|_F^2 .
\end{equation}
Minimizing this loss
is an equivalent formulation of the TVART problem.
The loss \eqref{eq:loss} is quadratic and convex in each of the 
variables $\U1, \U2$, and $\U3$, 
but it is not jointly convex in all variables.
This motivates our use of an alternating minimization algorithm
to minimize this cost;
this is also called coordinate descent or,
since the loss is quadratic, alternating least squares (ALS).
Minimizing over any one of the $\U{i}$ is a least squares
problem that can be solved by a linear matrix equation.
% Algorithm~\ref{alg:als1}.
In the next three subsections,
we compute the gradients of $L$ with respect to each set of variables
and discuss methods for solving the normal equations $\frac{\partial L}{\partial \U{i}} = 0$.
Alternating minimization is used in our final approach, 
Algorithm~\ref{alg:alt-min}, 
which incorporates regularization.
%Algorithms~\ref{alg:als1} and \ref{alg:alt-min}.

We will use the following useful identity mutiple times:
\begin{align}
\frac{\partial}{\partial \mat{B}} 
%\nabla_{\mat{B}} \,
\frac{1}{2} \left\| \mat{Y} - \mat{A} \mat{B}^\intercal \mat{C} \right\|_F^2
&=
\frac{\partial}{\partial \mat{B}} 
\frac{1}{2} \mathrm{Tr} \left[ 
\left( \mat{Y} - \mat{A} \mat{B}^\intercal \mat{C} \right)
\left( \mat{Y}^\intercal - \mat{C}^\intercal \mat{B} \mat{A}^\intercal \right)
\right]
\nonumber
\\
&=
-\mat{C} \mat{Y}^\intercal \mat{A} + 
\mat{C} \mat{C}^\intercal \mat{B} \mat{A}^\intercal \mat{A} . 
\label{eq:gradient_id}
\end{align}

\subsection{Left Spatial Modes $\U1$}

Taking partial derivatives of $L$ with respect to $\U1$,
we can employ \eqref{eq:gradient_id}
with
$
\mat{Y} \to \mat{Y}_k,
\mat{A} \to \mat{I},
\mat{C} \to \Dk \U2^\intercal \mat{X}_k
$
and
$\mat{B} \to \U1^\intercal$
and take the transpose to get
\begin{equation}
\label{eq:factor_1_ls}
\frac{\partial L}{\partial {\U1}}
=
\U1 \left( \sum_{k=1}^T 
\Dk \U2^\intercal \Xmat_k \Xmat_k^\intercal \U2 \Dk \right)
- 
\left( \sum_{k=1}^T 
\Ymat_k \Xmat_k^\intercal \U2 \Dk \right).
\end{equation}
Setting
$\frac{\partial L}{\partial {\U1}} = 0$
means we must solve
a matrix equation for $\U1$.
This requires forming an $N \times R$ matrix, 
and right multiplying by the pseudoinverse of an $R \times R$ matrix.

\subsection{Right Spatial Modes $\U2$}

We now differentiate $L$ with respect to $\U2$.
Replacing 
$
\mat{Y} \to \Ymat_k, 
\mat{A} \to \U1 \Dk, 
\mat{C} \to \Xmat_k$, and
$\mat{B} \to \U2$ 
in
\eqref{eq:gradient_id}
leads to the following:
\begin{equation}
  \label{eq:factor_2_ls}
  \frac{\partial L}{\partial {\U2}}
  = 
  \sum_{k=1}^T \Xmat_k \Xmat_k^\intercal \, \U2 \, \Dk \U1^\intercal \U1 \Dk
  -
  \sum_{k=1}^T \Xmat_k \Ymat_k^\intercal \U1 \Dk,
\end{equation}
which we can rewrite in the form
\begin{equation}
\label{eq:grad_2}
  \frac{\partial L}{\partial {\U2}}
  =
  \sum_{k=1}^T \mat{L}_k \U2 \mat{R}_k 
  - \mat{B},
\end{equation}
where 
$\mathbf{B} \in \R^{N \times R}$, 
$\mat{L}_k \in \R^{N \times N}$, and 
$\mat{R}_k \in \R^{R \times R}$.

We set $\frac{\partial L}{\partial {\U2}} = 0$,
so that \eqref{eq:grad_2} implies a Sylvester equation
for the optimal $\U2$:
\begin{equation}
\label{eq:sylv}
  \frac{\partial L}{\partial {\U2}}
  =
  0
  \iff
  \sum_{k=1}^T \mat{L}_k \U2 \mat{R}_k 
  = \mat{B} .
\end{equation}
For small $N$ and $R$, 
we can use the Kronecker product $\otimes$ and vectorization operations to 
rewrite \eqref{eq:sylv} 
as a linear vector equation
\begin{equation}
  \label{eq:sylv_vec}
  \mathrm{vec} 
  \left(
    \frac{\partial L}{\partial {\U2}}
  \right)
  = 0
  \iff
  \left( \sum_{k=1}^T \mat{R}_k^\intercal \otimes \mat{L}_k \right)
  \mathrm{vec} (\U2) 
  =  \mathrm{vec} (\mathbf{B}).
\end{equation}
This is a square $NR \times NR$ linear system in $NR$ unknowns.
However, for large $N$, it is impractical to even form this matrix.
Therefore, we solve the Sylvester equation \eqref{eq:sylv}
via the matrix-valued method of conjugate gradients (CG).
Note that we do not need to form the $N \times N$ matrices 
$\mat{L}_k = \mat{X}_k \mat{X}_k^\intercal$ 
to use matrix-free CG.

\subsection{Temporal Modes $\U3$}

Finally, we derive the gradients of $L$ with respect to $\U3$. 
The equations for $\U3$, through which we define the $\Dk$, 
are similar except in this case the unknown is a diagonal matrix.
The loss \eqref{eq:loss} with respect to $\Dk$ is decoupled across windows,
which leads to decoupled gradients as well.
Replacing 
$
\mat{Y} \to \Ymat_k, 
\mat{A} \to \U1, 
\mat{C} \to \U2^\intercal \Xmat_k$, and
$\mat{B} \to \Dk$
in
\eqref{eq:gradient_id}
leads to
\begin{align}
\frac{\partial L}{\partial {\Dk}}
&= 
\left( \U2^\intercal \Xmat_k \Xmat_k^\intercal \U2 
  \Dk
  \U1^\intercal \U1 % \right) * \mathbf{I}
-
%\left( 
\U2^\intercal \Xmat_k \Ymat_k^\intercal \U1 \right) * \mathbf{I}
\nonumber \\
&=
\left( \mat{L}_k' \Dk \mat{R}_k' - 
  \U2^\intercal \Xmat_k \Ymat_k^\intercal \U1 \right) * \mathbf{I}, 
\label{eq:factor_3_grad}
\end{align}
where $*$ is the elementwise or Hadamard product,
which enforces the diagonal constraint on $\Dk$ and its gradient,
and
$\mat{L}_k' = \U2^\intercal \Xmat_k \Xmat_k^\intercal \U2$ and
$\mat{R}_k' = \U1^\intercal \U1$
are $R \times R$ matrices.

By Theorem 2.5 in
\cite{million2007},
we have that
\[
 \left( \mat{L}_k' \Dk \mat{R}_k' \right)_{ii} = 
 \left( (\mat{L}_k' * \mat{R}_k') \left( \uvec3_{k:} \right)^\intercal \right)_i,
\]
where we have substituted $\Dk = \mathrm{diag}(\uvec3_{k:})$.
Thus the gradient can be written as
\[
\mathrm{vecdiag}
\left(
\frac{\partial L}{\partial {\Dk}}
\right)
= 
\left(\mat{L}_k' * \mat{R}_k' \right) \left( \uvec3_{k:} \right)^\intercal 
-
\mathrm{vecdiag} \left( \U2^\intercal\Xmat_k \Ymat_k^\intercal \U1 \right),
\]
where $\mathrm{vecdiag}(\mat{M})$ is the column vector formed 
by the diagonal elements of $\mat{M}$.
Thus, the least-squares problem for $\U3$ 
involves solving an $R \times R$ system
% \begin{equation}
% \label{eq:factor_3_ls}
% \mathrm{vecdiag} \left( \U2^\intercal \Xmat_k \Ymat_k^\intercal \U1 \right)
% =
% \left(\mat{L} * \mat{R} \right) \left( \uvec3_{k:} \right)^\intercal 
% \end{equation}
for each row $\uvec3_{k:}$ of $\U3$ independently, $k = 1, \ldots, T$.
However, adding regularization to the temporal modes
destroys this decoupling across time windows,
which is why we resort to CG or proximal gradient approaches
in Algorithm~\ref{alg:alt-min}.

\subsection{Computational complexity of alternating least squares}

Each minimization step requires the solution of a linear matrix equation
$\frac{\partial L}{\partial \U{d}} = 0$ 
for $\U{d}$.
We now comment on the difficulty of solving the 
subproblems in $\U1$, $\U2$, and $\U3$.
The subproblem in $\U1$ is by far the easiest, 
since it involves just one $R \times R$ system solve
or pseudoinverse.
The subproblem in $\U3$ requires
$T$ decoupled $R \times R$ system solves.
The cost of this step is thus linear in $T$.
We assume that $R$ is small enough that solving for each
time window is fast.
By far the most difficult step
is the subproblem
for the right hand factors $\U2$.
This challenge is in part because we obtain a Sylvester equation 
which we might naively solve using the Kronecker reformulation
as an $NR \times NR$ square system in $NR$ unknowns.
When we avoid using Kronecker products 
and instead use matrix-free CG,
we can be assured that this will take less memory
but could take more time, depending on the condition number
of the Kronecker product matrix \eqref{eq:sylv_vec}.
However, in either case the subproblem involves 
coefficient matrices of size $N$,
whereas for the other problems these are of size $R$.

\section{Regularization}
\label{sec:regularization}

% An important advantage of the tensor formulation of our
% problem is that it allows us to constrain the different components
% of the system matrices individually.
% We may have some prior knowledge about the structure of these factors
% that we can leverage. 
% A principled way in which to incorporate such information is through 
% {\it regularization}.

%\subsubsection{Norm regularization}

Alternating least squares is a
natural approach to the \ref{eq:TVART}
problem as formulated.
However, we have found that ALS is numerically unstable
for the switching linear test problem (Sec.~\ref{sec:test_problem}).
This instability is worst with low or zero observation noise,
but it is alleviated when larger noise added to the data.
Additive, independent noise 
adds a diagonal component to the data covariance, 
which suggests we apply Tikhonov regularization to the problem.
An additional motivation for Tikhonov regularization 
is that we do not want the entries in the matrices
$\U1$, $\U2$, and $\U3$ to become too large, 
but some might become large 
due to the scaling indeterminancy of the CP decomposition
\citep{kolda2009}.
We thus add a term 
\begin{equation*}
\frac{1}{2 \eta} \left( \| \U1 \|_F^2 + \| \U2 \|_F^2 + \| \U3 \|_F^2 \right)
\end{equation*}
to the least-squares loss.
The parameter $\eta$ controls the magnitude of this Tikhonov regularization:
larger $\eta$ is a smaller penalty, and smaller $\eta$ is a stronger penalty.
In matrix completion problems, a similar two-term regularization 
is often added and can be seen as a convex relaxation of the matrix rank.

\subsection{Temporal smoothing}

We also consider further regularization
of the temporal modes.
Recall that 
the rows of $\U3$ correspond to the loadings at different time windows.
By forcing these rows to be correlated, we keep the system matrices $\mat{A}_k$
from varying too much from window to window, a form of temporal smoothing.
The first regularizer we consider is a total variation (TV) penalty:
\begin{equation}
  \label{eq:TV}
  \mathrm{TV} (\U3) 
  = \sum_{r=1}^R \sum_{k=2}^{T} \big| u^{(3)}_{kr} - u^{(3)}_{k-1,r} \big|
  = \sum_{r=1}^R \| \mat{D} \, \uvec3_{:r} \|_1 .
\end{equation}
Matrix $\mat{D}$ is the $(T-1) \times T$ first difference matrix with free boundary conditions:
\begin{equation*}
  \mat{D} = 
  \left[
    \begin{array}{rrrrrr}
      1 & -1 &  &  & \ldots  & 0 \\
       & 1 & -1 &  &  & \vdots \\
       &  & \ddots & \ddots & & \\
      \vdots &  &  & 1 & -1 &  \\
      0 & \ldots & &  & 1 & -1
    \end{array}
  \right].
\end{equation*}
TV prefers piecewise constant time components $\uvec3_{:r}$
since it penalizes nonzero first-differences with column-wise $\ell_1$
penalty $\| \cdot \|_1$ to enforce sparsity, appropriate for an SLDS.
% The second regularizer we consider is a group total variation penalty
% \begin{equation}
%   \label{eq:groupTV}
%   \mathrm{GroupTV} (\U3) = \sum_{k=2}^{T} \left\| \uvec3_{k:} - \uvec3_{k-1,:} \right\|_2,
% \end{equation}
% which also prefers piecewise constant time components which change at the same time
% \citep{tank2017a}.
% We think of this as enforcing a switching linear dynamical system prior.

Alternatively, we consider a spline penalty:
\begin{equation}
  \label{eq:spline}
  \mathrm{Spline} (\U3) 
  = \frac{1}{2} \sum_{r=1}^R \sum_{k=2}^{T} \big( u^{(3)}_{kr} - u^{(3)}_{k-1,r} \big)^2
  = \frac{1}{2} \| \mat{D} \U3 \|_F^2.
\end{equation}
This linear smoother
penalizes the $\ell_2$-norm of the first derivative, leading to smoothly varying solutions.

% Finally, we may use the nonconvex and nonsmooth changepoint penalty:
% \begin{equation}
%   \label{eq:changept}
%   \mathrm{Changepoint}(\U3) 
%   = \sum_{r=1}^R \sum_{k=2}^{T} \mathds{1}_{\left\{ u^{(3)}_{kr} \neq u^{(3)}_{k-1,r} \right\}}\
%   = \| \mat{D} \U3 \|_0,
% \end{equation}
% where $\| \cdot \|_0$ is shorthand for the column-wise $\ell_0$ penalty.
% This regularization penalizes the number of changepoints directly and can be minimized
% using the methods of \citep{killick2012,jewell2019}.

\subsection{Regularized cost function}

We modify the problem \eqref{eq:TVART},
adding the Tikhonov and smoothing penalties 
to the loss function.
These additions result in the regularized cost function
\begin{equation}
  \label{eq:TVART*}
  C = 
  \frac{1}{2} \sum_{k=1}^T 
  \| \Ymat_k - \Amat_k \Xmat_k \|_F^2  
  + \frac{1}{2 \eta} \left( \| \U1 \|_F^2 + \| \U2 \|_F^2 + \| \U3 \|_F^2 \right)
  +
  \beta \,
  \mathcal{R}( \U3 ),
  \tag{TVART*}
\end{equation}
where $\Amat_k$ follows equation \eqref{eq:slice_modes} as before.
Here,
$\mathcal{R} (\cdot)$ is either TV$(\cdot)$ or
Spline$(\cdot)$.
Increasing the temporal smoothing strength $\beta$ 
leads to stronger regularization, as does decreasing $\eta$.

\subsection{Alternating minimization algorithm}

\begin{algorithm}[t!]
\caption{Alternating minimization for \ref{eq:TVART*}}
\label{alg:alt-min}
\begin{algorithmic}[1]
\STATE{initialize 
$\U1, \U2, \U3$ and
regularization parameters
$0 < \eta < \infty$,
$0 \leq \beta < \infty$ }
\REPEAT
\STATE{ 
$\U1 \leftarrow
\arg \min_{\mat{U}} 
C \left( \mat{U}, \U2, \U3 \right)$ 
\hfill $R \times R$ linear system solve}
%\STATE{normalize columns of $\U{1}$ and store norms as ${\bm \lambda}$}
\STATE{
$\U2 \leftarrow
\arg \min_{\mat{U}} 
C \left( \U1, \mat{U}, \U3 \right)$ 
\hfill
conjugate gradient}
%\STATE{normalize columns of $\U{2}$ and store norms as ${\bm \lambda}$}
\STATE{
$\U3 \leftarrow
\arg \min_{\mat{U}} 
C \left( \U1, \U2, \mat{U} \right)$
\hfill
conjugate gradient (Spline)\\\hfill or proximal gradient (TV)}
\UNTIL{convergence criteria}
\RETURN{$\U1, \U2, \U3$}
\end{algorithmic}
\end{algorithm}

We solve the regularized problem \ref{eq:TVART*} 
using alternating minimization,
also known as block coordinate descent,
detailed in Algorithm~\ref{alg:alt-min}.
The subroutines that minimize for $\U1, \U2$ and $\U3$
require different approaches.
Since the objective \ref{eq:TVART*} is quadratic
in $\U1$ and $\U2$,
we find these by solving a linear matrix equation
either directly or using CG;
CG works best for solving the Sylvester equation in $\U2$.
For $\U3$ with the Spline penalty, the cost is again quadratic so we also use CG.
However, the TV penalty is convex but not smooth, so in this case we use
the proximal gradient method with Nesterov acceleration \citep{beck2009}.

Algorithm~\ref{alg:alt-min}
has similar complexity to ALS for the unregularized problem.
However, in the regularized case the cost is dominated by 
the minimization over $\U3$,
which requires computing the proximal operator
at each step of the proximal gradient method.

The following theorem proves that when Algorithm~\ref{alg:alt-min} converges, 
it converges to a local minimum of the cost:
\begin{theorem}[Convergence of alternating minimization]
  \label{thm:convergence} 
  For a convex regularizer $\mathcal{R}$,
  the sequence of iterates generated by
  Algorithm~\ref{alg:alt-min} is defined and bounded, and every cluster point is a 
  coordinatewise minimum, i.e.\ a Nash point,
  of the regularized cost \ref{eq:TVART*}.
\end{theorem}
\begin{proof}
  We use the framework of \cite[][Theorem 5.1]{tseng2001}
  for cyclic block coordinate descent, 
  of which Algorithm~\ref{alg:alt-min} is an example.
  This theorem requires objective functions with
  convex and lower semicontinuous blocks
  as well as bounded level sets.
  We split the cost into smooth and non-smooth parts
  \begin{equation*}
    \label{eq:cost_split}
    C (\U1, \U2, \U3) =
    f_0 (\U1, \U2, \U3) + 
    f_1 (\U3) ,
  \end{equation*}
  where 
  $f_1 (\U3) = \beta \mathcal{R}(\U3)$
  and $f_0$ contains the remaining loss and Tikhonov terms.
  The function $f_0$ is continuous and differentiable, 
  and it is $\frac{1}{\eta}$-strongly convex in each of its blocks
  $\U1, \U2$, and $\U3$. 
  However, $f_0$ is not a convex function.
  Also, $f_1$ is convex and continuous.
  Let $(\mat{U}_0^{(1)}, \mat{U}_0^{(2)}, \mat{U}_0^{(3)})$
  be the initialization
  and $a = C(\mat{U}_0^{(1)}, \mat{U}_0^{(2)}, \mat{U}_0^{(3)})$.
  Denote the level set
  \[
  S_a = \{(\U1, \U2, \U3): C (\U1, \U2, \U3) \leq a \}.
  \]
  Then, since
  $
  C(\U1, \U2, \U3) \geq 
  \frac{1}{2\eta} 
  \left( \| \U1 \|_F^2 + \| \U2 \|_F^2 + \| \U3 \|_F^2 \right)$
  and
  the ball 
  $B_0(r) = \{x : \| \U1 \|_F^2 + \| \U2 \|_F^2 + \| \U3 \|_F^2 \leq r \}$
  is bounded,
  we can conclude that the level set
  $S_a \subseteq B_0(2\eta a)$
  is also bounded.
  Then by \cite[][Theorem 5.1]{tseng2001}, we obtain the result.
\end{proof}
\begin{remark}
  \label{remask_instability}
  We did not use any structure of $\mathcal{R}$ besides convexity 
  and lower semicontinuity,
  thus the same convergence results hold for 
  other regularizations with those properties.
  However, we did need to use the Tikhonov penalty to ensure
  that the level sets are bounded.
\end{remark}

% \todo{Convergence rate}

% \todo{Prove for nonconvex Changepoint}

\subsection{Implementation details}
\label{sec:alg_details}

The code and instructions for running it are available
from 
%\url{anonymized_url}.
\url{https://github.com/kharris/tvart}.
We implemented Algorithm~\ref{alg:alt-min} in MATLAB. 
It was run on an 
Intel(R) Xeon(R) CPU E5-2620 v4 @ 2.10GHz, 1200 MHz
with 32 cores and 128 GB RAM
using Ubuntu 16.04.1 with Linux kernel 4.15.0-48-generic and 
64-bit MATLAB R2017b.
The CG and prox-gradient subroutines are limited to 24 and 40 iterates,
respectively.
Proximal gradient uses a backtracking line search to find the step size
and the proximal operator of TV is evaluating using 
{\tt prox\_tv1d} from UNLocBox \citep{perraudin2014}, and
this is typically the bottleneck step for large problems.
All hyperparameter tuning (for $R$, $M$, $\eta$, and $\beta$) was performed manually.

For initialization we use the following method.
We consider the entire timeseries
in two snapshot matrices,
\begin{equation*}
\Xmat = \left[ 
  \begin{array}{cccc}
    \vert & \vert & & \vert \\
    \bm{x}(1) &\bm{x}(2) &\ldots &\bm{x}(T M) \\
    \vert & \vert & & \vert 
  \end{array}
\right]
%\]
%and
%\[
\quad \mbox{and} \quad
\Ymat = \left[ 
  \begin{array}{cccc}
    \vert & \vert & & \vert \\
    \bm{x}(2) &\bm{x}(3) &\ldots &\bm{x}(TM+1) \\
    \vert & \vert & & \vert 
  \end{array}
\right] .
\end{equation*}
We then fit a single linear model to the timeseries
by 
$\Amat = \Ymat \Xmat^\dagger$ and form the matrices
$[ \mat{U}, \mat{S}, \mat{V} ] = \mathrm{svd}(\Amat)$.
When $R = N$, the TVART modes are then initialized to
$\mat{U}_0^{(1)} = \mat{U}$,
$\mat{U}_0^{(2)} = \mat{V}$,
and
$\mat{U}_0^{(3)} = \vec{1}_T \vec{1}_R^\intercal/\sqrt{T}$ is
the matrix of all-ones with columns normalized.
If $R < N$, we truncate the smaller singular vectors,
and when $R > N$, we add columns equal to 
$\vec{1}_N / \sqrt{N}$ to $\mat{U}_0^{(1)}$ and $\mat{U}_0^{(2)}$
and $\vec{1}_T/\sqrt{T}$ to $\mat{U}_0^{(3)}$.
Initializations with unequal columns appear to help speed up the initial phase
of the optimization, as opposed to all-ones.
For this reason, we add Gaussian noise to the initializations
with standard deviation $0.5 / \sqrt{N}$ to $\mat{U}_0^{(1)}$ and $\mat{U}_0^{(2)}$
and with standard deviation $0.5 / \sqrt{T}$ to $\mat{U}_0^{(3)}$.
Initializing to zeros is not appropriate since the gradients in that case
are always zero.
All of the results we present do not depend strongly on the realization of the noise.

For stopping criteria, 
we use both a relative and absolute tolerance
on the decrease in the cost function.
Let $C_t$ be the cost of \ref{eq:TVART*} at iterate $t$.
Then, we stop if either
\begin{equation}
  \label{eq:convergence}
  \frac{|C_t - C_{t-1}|}{C_{t-1}} < {\tt rtol} \qquad {\rm or} \qquad |C_t - C_{t-1}| < {\tt atol}.
\end{equation}
Unless specified otherwise, {\tt rtol} = $10^{-4}$ and {\tt atol} = $10^{-6}$.
In all cases we have tested, the relative tolerance is acheived first.
We report the cost \ref{eq:TVART*} as it runs
as well as the root mean square error (RMSE), 
which we define as:
\begin{equation}
  \label{eq:rmse}
  {\rm RMSE} = \sqrt{ \frac{1}{N M T} \sum_{k=1}^T \| \mat{Y}_k - \mat{A}_k \mat{X}_k \|_F^2} \, .
\end{equation}
The RMSE is normalized 
so that it gives a measure of average one-step prediction error per channel.
This ``goodness of fit'' metric 
can then be compared to the standard deviation of the data.

% \subsection{Slackness parameter schedule}

% We have found that it is best to iterate
% until convergence then decrease parameter
% $\eta_2$ until $\mat{W}$ and $\U3$ are close enough.
% Let 
% \[
% w(t) = {\| \mat{W} - \U3 \|_F} / {\|\mat{W}\|_F}
% \]
% be the relative difference.
% Starting from an initial $\eta_2$,
% we iterate run Algorithm~\ref{alg:alt-min}
% until the convergence criteria \eqref{eq:convergence}
% are met.
% If $w(t) > 10^{-4}$, then we set 
% $\eta_2 /10 \to \eta_2$.
% If $c(t) < c_{\rm ref} / 10$, 
% then we also update the reference cost
% $c(t) \to c_{\rm ref}$
% to ensure that the algorithm takes sufficient minimization steps.

% \section{Algorithm implementation}

% We have implemented Algorithm~\ref{alg:alt-min} in MATLAB.

% \subsection{Rescaled cost function}

% \begin{equation}
%   \label{eq:cost_final}
%   \frac{C \left( \U1, \U2, \U3 \right)}{N M T}  +
%   \frac{\beta}{R} \,\mathcal{R}(\U3) + 
%   \frac{1}{2 \eta R} \left( \frac{\| \U1 \|_F^2}{N} + \frac{\| \U2 \|_F^2}{N} +  \frac{\| \U3 \|_F^2}{T} \right)
% \end{equation}

\section{Example applications}

In this section we test our method on both synthetic data and real-world datasets.
With synthetic data generated by a switching or smoothly varying linear dynamical system, 
we show that TVART can recover
the true dynamics and is competitive with other state-of-the-art techniques.
In real-world data examples, we highlight how the recovered modes
are interpretable and can correspond to 
important dynamical regimes.
For a complete list of the parameters used, please refer to Appendix~\ref{app:examples}.

\subsection{Test problem 1: switching low rank linear system}
\label{sec:test_problem}

\begin{figure}[t!]
  \centering
  %\vspace{-3cm}
  % \includegraphics[width=0.6\linewidth,trim=15 10 15 10, clip]{../figures/switching_summary.eps}\hfill
  % \includegraphics[width=0.2\linewidth,trim=10 0 10 0, clip]{../figures/switching_components.eps}\hfill
  % \includegraphics[width=0.2\linewidth,trim=10 0 10 0, clip]{../figures/switching_matrices.eps}
  \includegraphics[width=\linewidth]{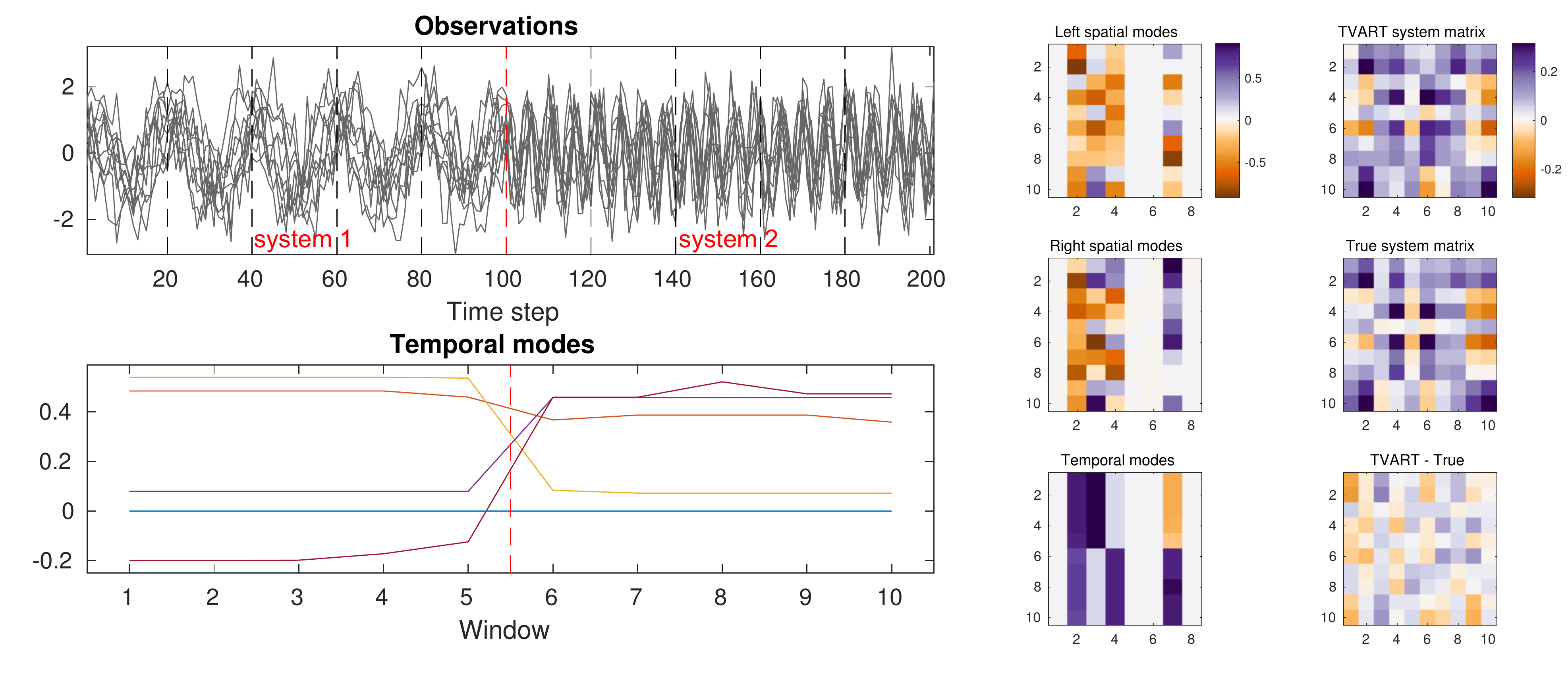}
  \caption{
    TVART correctly recovers the system matrices in a switching linear test problem. 
    {\bf (Left, top)} Noisy observations input into the algorithm,
    where each line is one spatial measurement in time.
    Dashed lines denote the windows.
    {\bf (Left, bottom)}
    The temporal modes from TVART per window, 
    which clearly pick out the change point at window 6.
    {\bf (Center)} We show the left and right spatial modes and
    the temporal modes output by TVART on the switching linear test case.
    Only four components are significantly different from zero.
    {\bf (Right)} 
    In the top, we show the TVART estimate of the system matrix $\mat{A}_1$,
    the middle shows the truth, and the bottom shows their difference,
    which is relatively small. Colorbars are shared in the center and right columns.
  }
  \label{fig:test_case}
\end{figure}
% \begin{figure}
%   \centering
%   \includegraphics[width=0.49\linewidth]{../figures/switching_components.eps}\hfill
%   \includegraphics[width=0.49\linewidth]{../figures/switching_matrices.eps}
%   %\includegraphics[width=0.49\linewidth]{../figures/switching_components.png}\hfill
%   %\includegraphics[width=0.49\linewidth]{../figures/switching_matrices.png}.
%   \caption{    
%     (Left) We show the left and right spatial modes and
%     the temporal modes output by TVART on the switching linear test case.
%     Note that only four components are significantly different from zero.
%     (Right) We depict the system matrix $A$ at the first window.
%     In the top row, we show the TVART estimate of this matrix,
%     the second row shows the truth, and the third row shows their difference.
% }
%   \label{fig:test_case_modes}
% \end{figure}

\begin{figure}[t]
  \centering
  \includegraphics[width=0.6\linewidth]{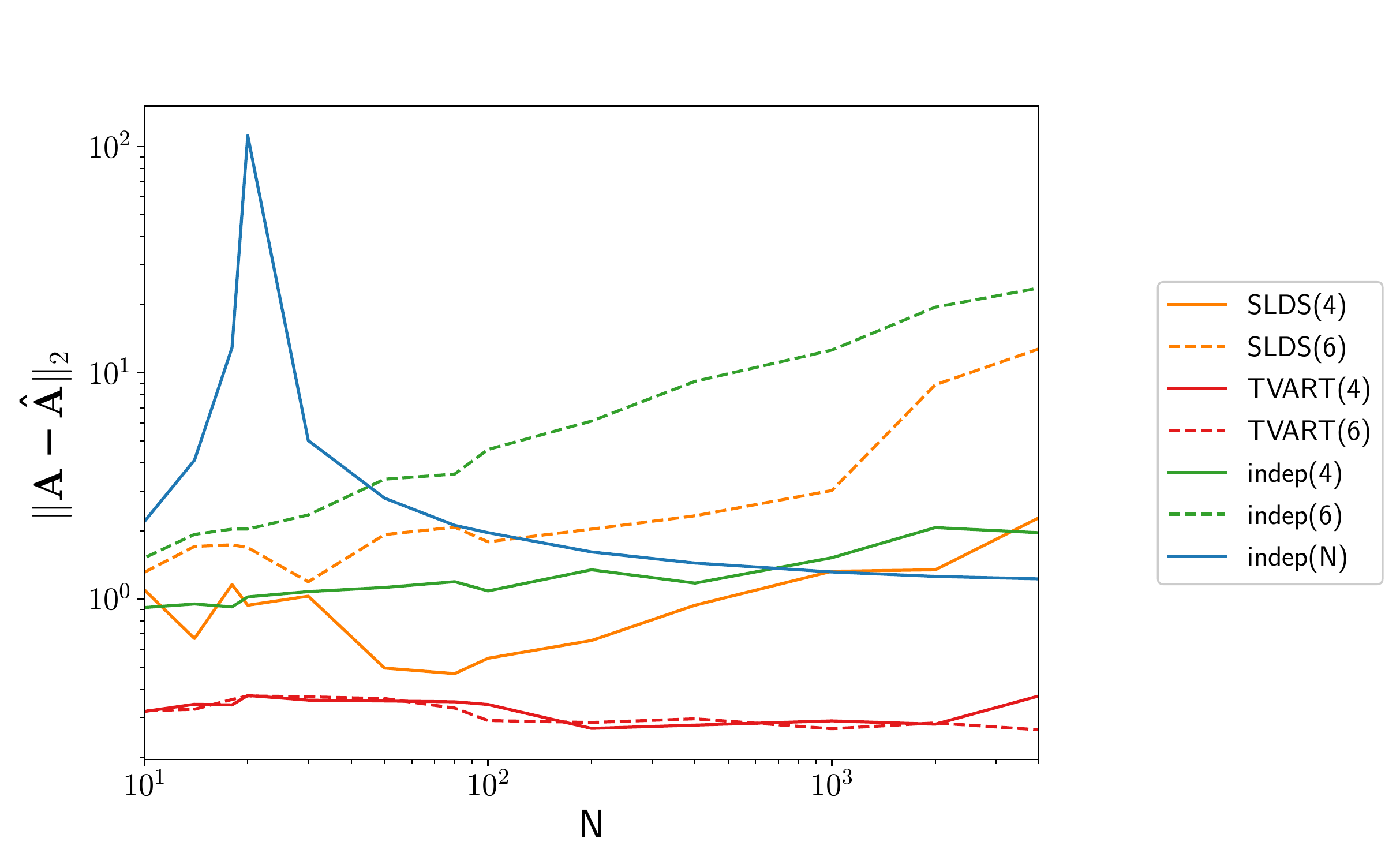}
  % \hfill
  % \includegraphics[width=0.49\linewidth,trim={.4cm .5cm 1.4cm 1.1cm}]{../figures/switching_compare_runtimes.eps}
  \caption{TVART outperforms other methods on the switching test problem, for varying system size $N$.
  %{\bf (Left)} 
  Average error of the inferred system matrix $\hat{\mat{A}}$ versus the true
  $\mat{A}$ measured in operator norm.
  %{\bf (Right)} Runtimes of TVART versus SLDS from the package {\tt ssm}.
  %TVART is able to acheive similar performance as SLDS with much lower computational cost.
}
  \label{fig:test_comparison}
\end{figure}

\subsubsection{Model and data generation}

We first apply TVART to a simple test case 
where the true model is a low rank,
switching linear system.
We generate two $N \times N$ system matrices
$\mat{A}_1$ and $\mat{A}_2$,
which are random, rank-2 rotation matrices.
For the first half of our timeseries, the dynamics follow
$\mat{A}_1$, and then they switch to $\mat{A}_2$.
Specifically, we form the matrices 
$\mat{A}_i$ for $i=1, 2$ by the following process:
\begin{enumerate}
\item Generate a $2 \times 2$ rotation matrix
  $\mat{A}_i' = \left(
    \begin{array}{lr}
      \cos(\theta_i) & -\sin(\theta_i) \\
      \sin(\theta_i) & \cos(\theta_i)
    \end{array}
  \right)$.
\item Draw a random matrix $\mat{Z}_i \in \R^{N \times 2}$, 
  with standard Gaussian entries.
\item Take the SVD: 
  $[\mat{W}_i, \mat{S}_i, \mat{V}_i] = \mathrm{svd} (\mat{Z}_i) $.
\item Form a larger rotation matrix by projecting with onto the 
  left singular vectors: 
  $\mat{A}_i = \mat{W}_i \mat{A}_i' \mat{W}_i^\intercal$.
\end{enumerate}
Then, $\mat{A}_i$ is a random rotation matrix of rank 2. 
We use $\theta_1 = 0.1\pi$ and $\theta_2 = 0.37 \pi$.

The trajectory is generated by starting with the initial condition 
$\vec{x}(1) = \vec{1}$
and iterating 200 steps with $\mat{A}_1$ to remove a transient.
After the transient, the vector is renormalized to have $\|\vec{x}(t)\| = \sqrt{N}$.
For the first half of the timeseries $t < \tau/2$, 
we iterate $\vec{x}(t+1) = \mat{A}_1 \, \vec{x}(t)$,
while for the second half $t \geq \tau/2$ 
we iterate $\vec{x}(t+1) = \mat{A}_2 \, \vec{x}(t)$.
After the first application of $\mat{A}_2$, 
we again rescale $\vec{x}(t)$ to have norm $\sqrt{N}$.

After generating the noiseless trajectory, 
we add independent, Gaussian observation noise
with standard deviation $\sigma$ to each entry of $\vec{x}(t)$.
Since the noise vector norm scales as $\sqrt{N}$, rescaling the signal vector by the 
same factor allows us to maintain the same signal-to-noise ratio across different 
system sizes $N$.
The RMSE of a model that is not overfit should be approximately $\sigma$ for any $N$.

\subsubsection{Performance of TVART}

We present results in 
Figure~\ref{fig:test_case} for $N = 10$ and $\sigma=0.5$.
The TVART algorithm is run with 
$M=20$, $R = 8$, $\eta = 1/N$,
$\beta = 5$, and TV regularization.
Algorithm~\ref{alg:alt-min} converges in 30 iterations
with an RMSE of 0.554, which is close 
to the noise floor $\sigma = 0.5$.
The temporal modes are stable for the first 5
windows, when system $\mat{A}_1$ is active,
%exhibit a change in window 6 when the switch occurs,
then switch to a different state for the remaining windows
where $\mat{A}_2$ is active.
%If the switch does not exactly occur on the window boundary, little changes.
Thus just from the temporal modes, we can determine the change point
to the resolution of the window size.
Examining the TVART output in more detail,
only the first four 
spatial and temporal modes are significantly different from 0.
Remarkably, TVART is able to discover the true rank of the system;
this is 4 because $\mat{A}_1$ and $\mat{A}_2$ are each rank 2.
Furthermore, we compare the TVART reconstruction of $\mat{A}_1$
to the truth; the reconstruction matches the truth very closely
(Figure~\ref{fig:test_case}).
% this is true of all windows except for window 6.
% In window 6, the change point occurs halfway through,
% so their is no consistent linear system that
% fits the whole window.

We now describe some other behaviors that have been observed
in this test problem for parameters that are not shown.
When the noise $\sigma$ is very small or zero,
the Tikhonov regularization term is important.
If we eliminate this term by taking $\eta \to \infty$,
the ALS subproblems become ill-conditioned and lead to poor results.
Without the Tikhonov penalty,
the true rank of the system is also less obvious---the 
entries which are nearly 0 in Figure~\ref{fig:test_case} are noisier---whereas
the penalty shrinks these noisy entries.
The TV regularization, in contrast,
is important for selecting a switching solution under moderate to large noise.

\subsubsection{Comparison to other methods}

We compared TVART method to two alternative approaches: 
independent windowed fits using multivariate regression,
for a simple baseline,
as well as a Bayesian SLDS fit using the 
{\tt ssm} package available from
\url{https://github.com/slinderman/ssm}.

The independent model is tested with full rank and
rank truncation to 4 or 6, using the SVD.
To fit a rank $R<N$ model on the $k$th window of data,
we first perform the SVD on the entire $k$th timeseries,
fit a linear model for the $R$ leading temporal modes,
and project back into the $N$-dimensional space using the spatial modes.

The SLDS fit was allowed 6000 iterations and
performed with either a 4 or 6 dimensional latent space
with the maximum number of switches parameter set to {\tt Kmax} = 4
and otherwise default parameter values.
% For the runtime calculation of SLDS, we:
% \begin{enumerate}
% \item First fit the SLDS for 6000 iterations,
% \item Took the evidence lower bound (ELBO) sequence and 
%   smoothed it using a 2nd-order lowpass Butterworth filter
%   with cutoff frequency 0.01 per iteration, 
% \item Fit a sigmoid to the smoothed ELBO, and
% \item Chose the final iteration number as the estimated step at which the ELBO 
%   reaches 99.9\% of its maximum according to the sigmoid.
% \end{enumerate}
% The SLDS was then refit for that number of iterations to compute the run time.

For TVART we use a rank of either 4 or 6 and the parameters
$M=20, \eta = 1/N$, $\beta = 1$, and TV regularization.
Recall that the true rank of the system is 4.

In Figure~\ref{fig:test_comparison},
we show the results as we sweep across system sizes
$N$ from 10 to 4000.
We depict the error in the reconstruction
$\| \mat{A} - \hat{\mat{A}} \|_2$ averaged across all time windows.
Note that if we measure the entrywise error in the reconstruction,
this decreases with $N$, whereas the operator norm error is constant;
the relative trends remain the same in any norm.
We see that TVART of either rank is able to recover the true dynamics $\mat{A}$ 
better than any of the other methods.
SLDS performs worse but slightly better than independent for rank 4,
rank 6 is similar,
although both degrades for higher system sizes $N$.
TVART is also much faster than SLDS, by roughly an order of magnitude.

These results are with the current version of \texttt{ssm} with the default settings.
A previous version of this package with similar settings did perform better
(results are shown in an earlier version of this paper, available at
\url{https://arxiv.org/abs/1905.08389v1}).
However, in that case, TVART was still as good or better than the alternative methods
while running much faster.
We conclude that TVART with TV regularization
is a scalable, alternative way of finding low rank switching linear dynamics.

\subsection{Test problem 2: smoothly varying low rank linear system}

\subsubsection{Model and data generation}

Another comparison was made by assuming rank 2 orthogonal dynamics that are {\em smooth}.
In the previous switching linear example, the dynamics were generated from 
a rotation matrix for two different angles $\theta_1$ and $\theta_2$.
In this case, the angle
$\theta(t)$ comes from a smooth Gaussian process 
in order to have slowly-varying rank 2 dynamics.
In this case, the orthogonal projection into the higher space $\mat{W}$ is 
fixed for all time.

In detail, we form a different system matrix at each time step $t$ by
\begin{equation*}
  \mat{A}(t) = 
  \mat{W}  
  \left(
    \begin{array}{lr}
      \cos(\theta(t)) & -\sin(\theta(t)) \\
      \sin(\theta(t)) & \cos(\theta(t))
    \end{array}
  \right)
   \mat{W}^\intercal ,
\end{equation*}
where the random basis $\mat{W}$ is drawn as before.
The angle $\theta(t)$ is from a centered Gaussian process
with covariance function
\begin{equation*}
  K(t,t') = 
  \exp \left[
    -\left(\frac{t - t'}{30}\right)^2
  \right]
  +
  0.001 \,\delta_{t,t'} .
\end{equation*}
This squared exponential covariance produces smooth trajectories $\theta(t)$
that are approximately constant over the timescale of $\sim 30$ time steps.
The small diagonal covariance ensures numerical stability of the Cholesky decomposition.
We add Gaussian observation noise with standard deviation $\sigma = 0.2$.

\subsubsection{Results}

We depict the data for a realization with $N=10$ in Figure~\ref{fig:test_case_2}.
This includes the angle $\theta(t)$ drawn from the Gaussian process and
observations of the dynamical system output.
It is clear that the frequency of oscillations in the state variables
are stable over short timescales, but that the frequency is modulated as $\theta(t)$ changes.

TVART is applied to these data with $R=4$, 
$M = 1$,
$\eta = 6 / N$,
and $\beta = 600 \, \log_{10}^2 (N)$ with the Spline regularization.
The temporal modes are shown as well as a comparison of
$\mat{A}(t)$ versus the inferred $\hat{\mat{A}}(t)$
at $t=1$.
Interestingly, with these values of regularization, 
we can fit an accurate model at every time point without overfitting.
The recovered system is effectively rank 3 selected by the Tikhonov penalty.

We also tested the method for other values of $N$ varying between 6 and 4000.
The results are not shown, but the performance of TVART is stable
across system sizes, with $\| \mat{A} - \hat{\mat{A}} \|_2 \approx 0.1$.
Independent fits perform significantly worse, and SLDS is even poorer.
This is not surprising, since an independent model is sure to overfit while
the SLDS model is inherently not smooth and thus highly biased.

\begin{figure}[t!]
  \centering
  \includegraphics[width=0.59\linewidth,trim={0cm .5cm 0cm .2cm},clip]{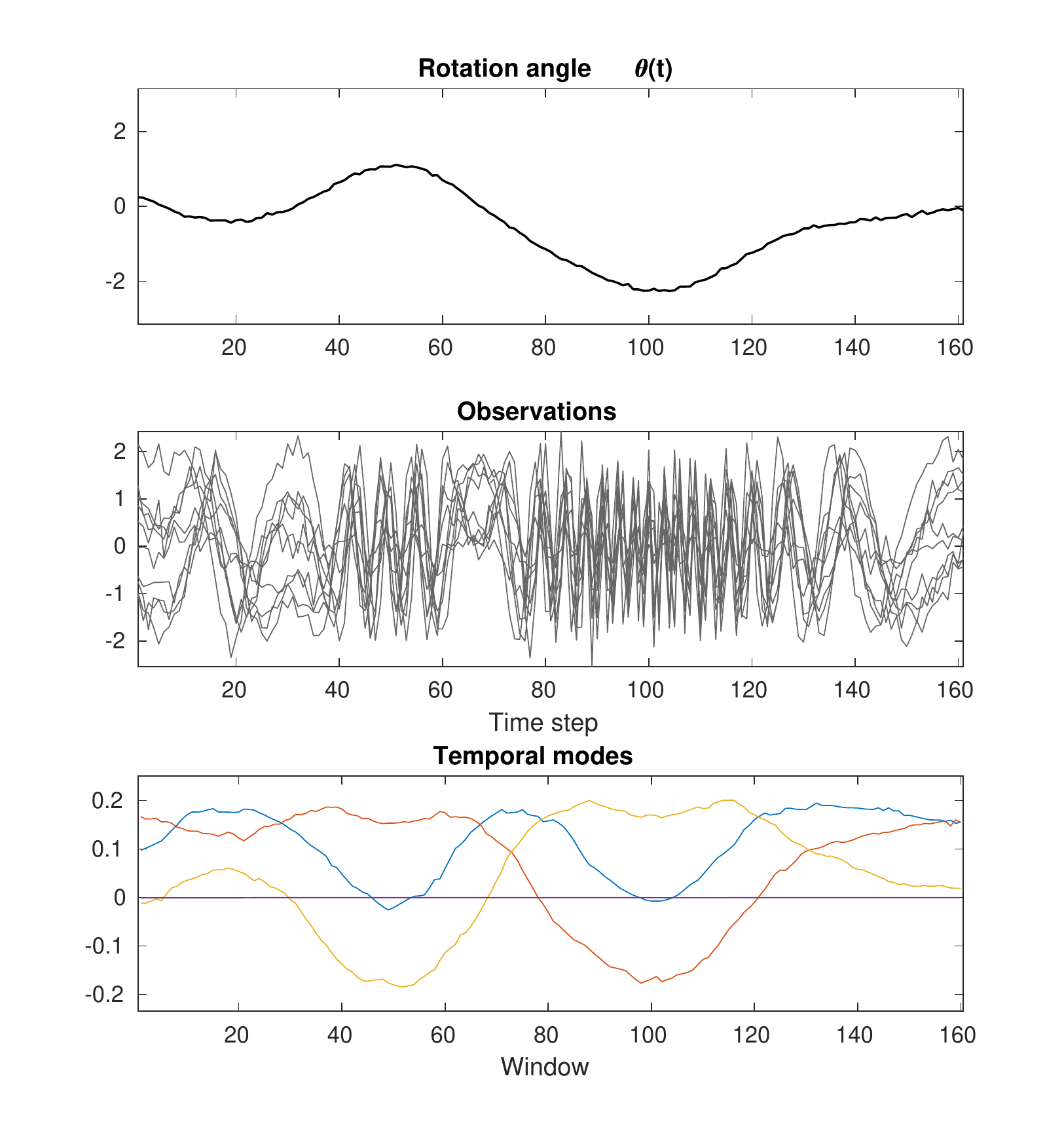}
  \hspace{.5cm}
  \includegraphics[width=0.25\linewidth]{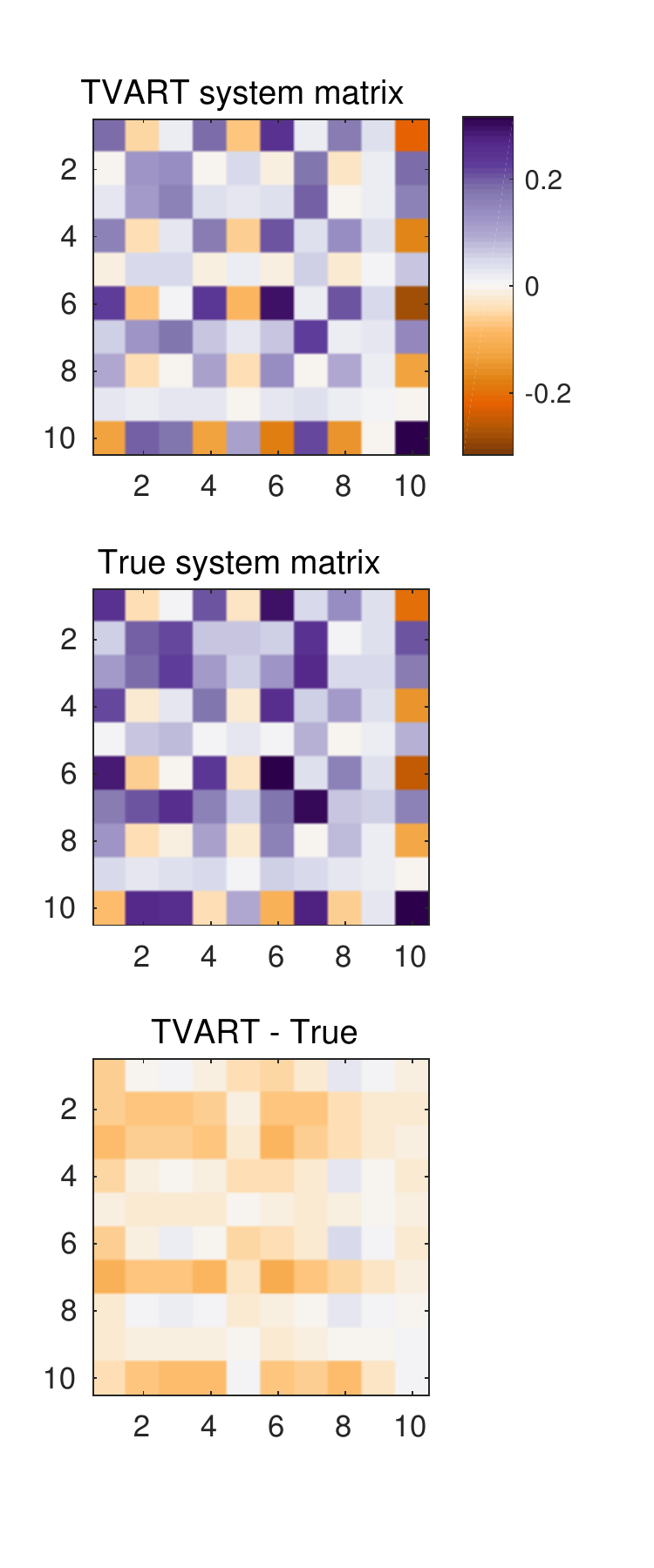}
  % \hfill
  % \includegraphics[width=0.49\linewidth,trim={.4cm .5cm 1.4cm 1.1cm}]{../figures/smooth_compare_err_inf.eps}
  \caption{
    Smoothly varying linear test case. 
    {\bf (Left, top)} The rotation angle trajectory $\theta(t)$ drawn from 
    the Gaussian process.
    {\bf (Left, middle)} State variable observations for an example fit with $N = 10$.
    {\bf (Left, bottom)} The temporal fit with the parameters described in the text. 
    For the strong temporal regularization used,
    we can fit an accurate model at every time point.
    {\bf (Right)} Comparing the inferred matrix and the truth for the first timestep
    $t = 1$.}
  \label{fig:test_case_2}
\end{figure}

\subsection{Dataset 1: worm behavior}

\begin{figure}[t]
  \centering
  \includegraphics[width=.6\linewidth,trim={1.2cm 1cm 1.3cm 1cm},clip]{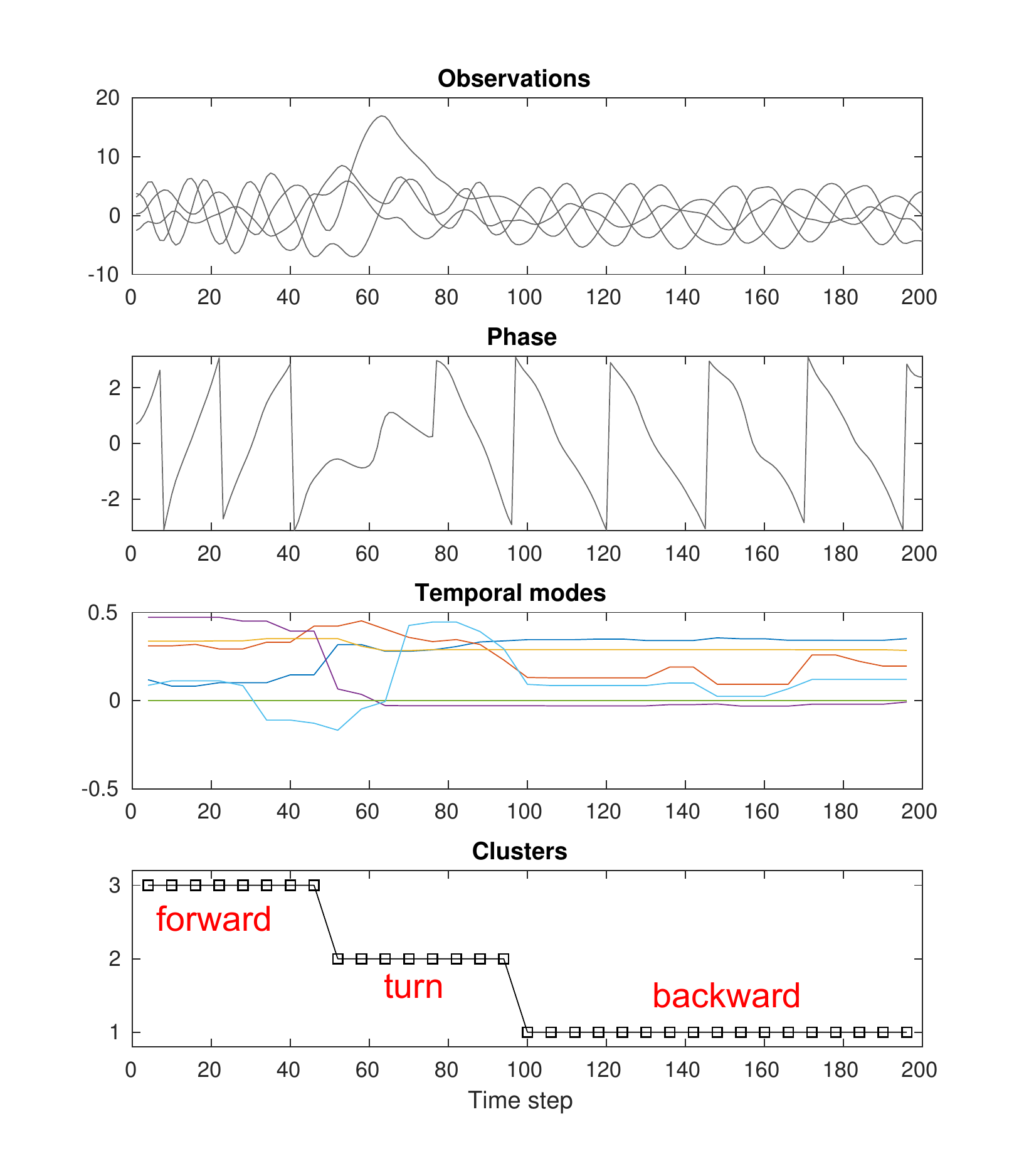}
  \caption{TVART applied to the worm behavior dataset.
    In this example, the worm begins moving forward, executes a turn, 
    and continues but in the backward direction.
    The temporal modes in TVART are able to pick out these three regimes,
    and a clustering of them identifies these three behavioral states.
  }
  \label{fig:worms}
\end{figure}

We analyze the escape response behavior of the nematode worm 
{\it Caenorhabditis elegans}
in response to a heat stimulus \citep{broekmans2016a,broekmans2016}.
These were used as test data for clustering via adaptive linear models
in the recent work of \cite{costa2019}.
Worm postural data were analyzed as smooth timeseries of
$N = 4$
``eigenworm'' principal components.
We ran TVART with an affine model and $R=6, M=6, \eta = 0.05$, and $\beta = 6$ on these data.
We also compared the performance of the code provided by \cite{costa2019}
at \url{https://github.com/AntonioCCosta/local-linear-segmentation}.

The results for worm 1 are shown in Figure~\ref{fig:worms}. 
We performed clustering on the system matrices as before with three clusters
and found that these clusters matched the three behaviors in the data:
forward crawling, a turn, and backward crawling.
These clusters are not trivial:
the data means during forward and backward motion are approximately equal, 
but the phase velocity switches.

Finally, we also compared our results to the code provided by
\cite{costa2019}, which fits an adaptive linear model and clusters the same timeseries;
the clustering results are essentially the same.
In terms of runtime, fitting a TVART model, performing clustering, and displaying the results
takes 23 s (90 iterates), versus 123 s for the code of \cite{costa2019}.
Thus, we see that our method is much faster than theirs,
while producing essentially the same clustering results.

\subsection{Dataset 2: sea surface temperature}

\begin{figure}[t]
  \centering
  \includegraphics[width=0.6\linewidth,trim={1cm 2cm 1.5cm .1cm},clip]{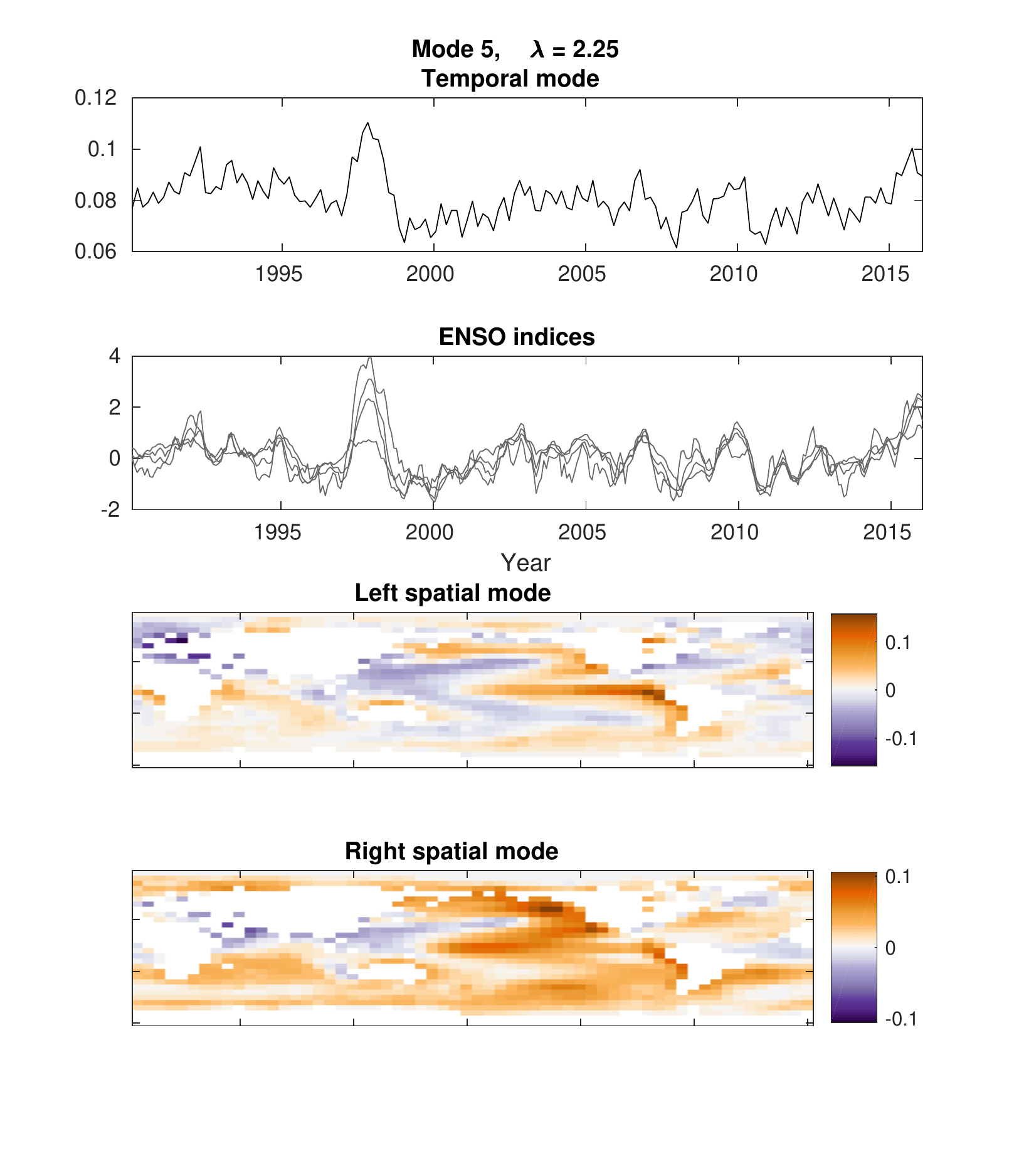}
  \caption{Mode corresponding to the ENSO 
    (El Ni\~no/Southern Oscillation) climate pattern 
    found with TVART.
    ENSO is associated with warm water in the equatorial Pacific Ocean.
  }
  \label{fig:sst}
\end{figure}

We applied our method to weekly sea surface 
temperature data from 1990 until present
\citep{reynolds2002}.
Sea surface temperature data contain oscillations with varying timescales,
from seasonal to multi-decadal.
This is also a very high-dimensional test for the TVART algorithm.

Weekly sea surface temperature grids
were downsampled by a factor of 6 in the latitudinal 
and longitudinal directions,
resulting in final vectors of length $N = 1259$.
We chose a window size of $M = 9$ weeks, 
resulting in $T = 171$ windows, and use parameters $\eta = 10^{-3}$, $\beta = 10^4$,
and the Spline regularizer.
Using our standard initialization, the ALS routine stagnates 
as the matrix $\U2$ converges very slowly.
However, we have found that restarting the algorithm after 10 iterations
and setting $\U2 = \U1$ in the new initialization speeds things up significantly.
This heuristic was inspired by the fact that the solution has $\U2 \approx \U1$.
In the end, the algorithm requires 1176 iterations to converge.

The leading modes that are output by the algorithm oscillate seasonally. 
However, we also find that mode 5, as ordered by $\ell_2$ energy, tracks 
the El Ni\~{n}o-Southern Oscillation (ENSO), Figure~\ref{fig:sst}.
The corresponding spatial modes show a plume of warm water in the 
central and eastern equatorial Pacific Ocean.
Warmer than average water in this location is the signature of ENSO.
Thus, TVART is able to discover an important dynamical feature in a large spatiotemporal dataset.

% \begin{figure*}[t!]
% %\begin{minipage}
%   \begin{minipage}[t]{0.49\linewidth}
%     \begin{figure}[H]
%       \centering
%       \includegraphics[width=\linewidth,trim={1cm 2cm 1.5cm .1cm},clip]{../figures/example_el_nino_mode_5.eps}
%       \caption{Mode corresponding to the ENSO 
%         (El Ni\~no/Southern Oscillation) climate pattern 
%         found with TVART.
%         ENSO is associated with warm water in the equatorial Pacific Ocean.
%       }
%       \label{fig:sst}
%     \end{figure}
%   \end{minipage}
%   \hfill
%   \begin{minipage}[t]{0.49\linewidth}
%     \begin{figure}[H]
%       \centering
%       \includegraphics[width=\linewidth,trim={1cm 1cm 1cm .5cm},clip]{../figures/neurotycho_clusters.eps}
%       \caption{TVART applied to neural activity (ECoG) data captures movements.
%         The dynamical clusters inferred by our method approximately correspond 
%         to movement and rest states.
%       }
%       \label{fig:monkey}
%     \end{figure}
%   \end{minipage}
% \end{figure*}

% \begin{figure}[ht!]
%   \centering
%   \includegraphics[width=\linewidth]{../figures/example_pdo_mode_10.eps}
%   \caption{Mode 12}
%   \label{fig:sst}
% \end{figure}

\subsection{Dataset 3: neural activity during a reaching task}

\begin{figure}[t]
  \centering
  \includegraphics[width=0.6\linewidth,trim={1cm 1cm 1cm .5cm},clip]{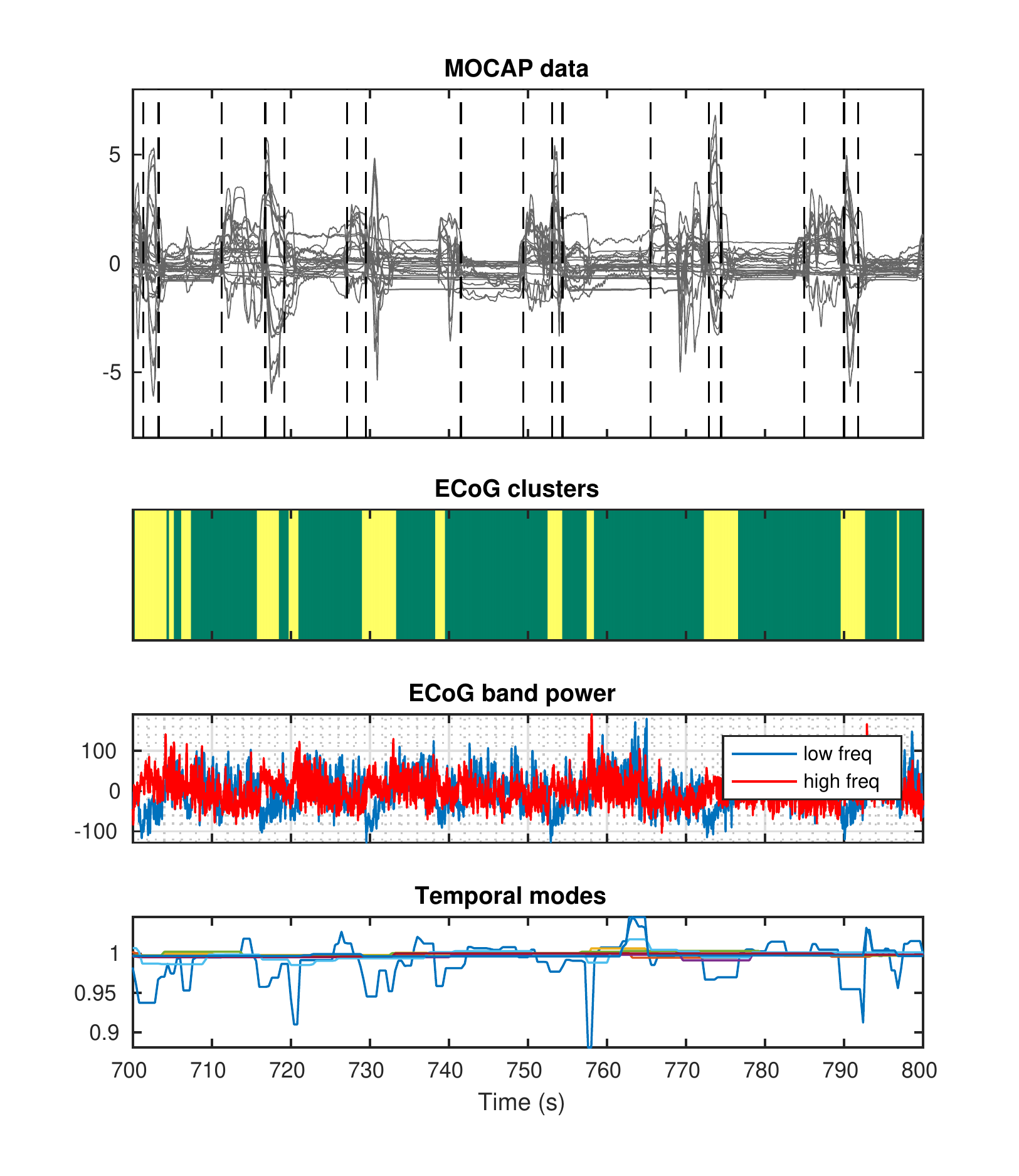}
  \caption{TVART applied to neural activity (ECoG) data captures movements.
    The dynamical clusters inferred by our method approximately correspond 
    to movement and rest states.
  }
  \label{fig:monkey}
\end{figure}

A second high-dimensional dataset comes from
electrocorticography recordings 
(ECoG)
of a Japanese macaque monkey {\it Macaca fuscata}
during a reaching task, provided by the NeuroTycho project \citep{chao2010}.
This is an invasive technology that measures brain activity 
using chronic electrodes placed below the skull and the dura mater.
During the task, the monkey makes repeated reaches towards food while its limbs
are tracked using motion capture and brain activity is recorded.

We analyzed the data of monkey K1, date 2009-05-25.
We ran TVART on the $N=64$ channel
ECoG voltage data after filtering out 50 Hz line noise,
downsampling to 500 Hz,
and standardization.
The TVART parameters were
$M = 200$, $R = 8$, $\eta = 1$, $\beta = 100$, and TV regularization,
resulting in $T = 1999$ windows of length 0.4 s.

Figure~\ref{fig:monkey} depicts the results:
We show motion capture data with dashed lines at the onset and offsets of
movements, using a changepoint detection procedure. 
We also highlight the result of clustering the TVART temporal modes, 
obtained from the brain activity alone, into two clusters.
The brain activity reveals two dominant modes, one of which is aligned to movement.
These movements are accompanied by an increase in high frequency (32-200 Hz)
and a decrease in low frequency (2-32 Hz) power.
The dominant TVART mode (show in blue) follows this spectral change in the timeseries.
TVART tracks this spectral change directly from the timeseries of electrode voltage.
Furthermore, the spatial modes associated with it are centered in the premotor region,
as depicted in Figure~\ref{fig:monkey_mode}.
Again, we observe that an important spatiotemporal feature has been extracted via
the TVART algorithm.

\begin{figure}[t]
  \centering
  \includegraphics[width=0.9\linewidth]{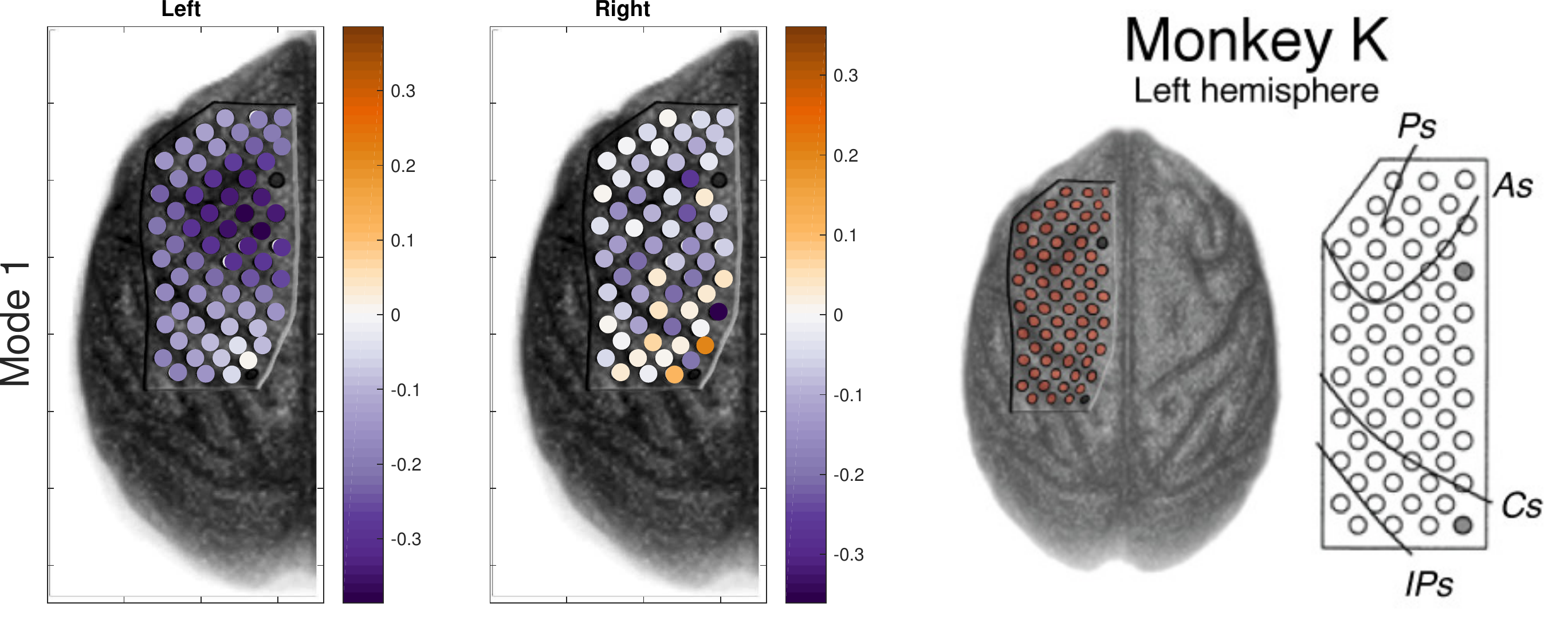}
  \caption{The dominant spatial modes (left and right) of TVART applied to the neural activity dataset.
    In the left mode, the electrodes with largest weight are centered in the premotor region
    known to be active during reaching,
    posterior of the arcuate sulcus and more medial than lateral.
    The right spatial mode is less interpretable.
    Data and brain images provided by NeuroTycho project \citep{chao2010}.
  }
  \label{fig:monkey_mode}
\end{figure}

\section{Conclusions}

We have presented a low rank tensor formulation of dynamical linear modeling
that we call {\it time-varying autoregression with low rank tensors}, or TVART. 
Our method offers the advantage of being able to scale to problems
with many state variables or many time points, 
as highlighted by our examples.
It also allows for incorporating prior knowledge of the temporal
structure of the dynamics via regularization,
which can lead to more stable and accurate reconstructions from less data.
We have shown that enforcing different kinds of temporal smoothness aids
in the identification of switching or slowly-varying dynamics.
We have also shown that TVART can be used with a variety of real datasets,
and the modes output by our method are often interpretable
as dominant dynamical features in the data.

There are limitations to TVART.
First, it has a number of hyperparameters,
such as the latent rank, stopping tolerances, and regularization strengths,
that must be tuned to the problem at hand.
While our method is generally more scalable than Bayesian approaches,
it suffers from fragility to hyperparameter choice.
In our experiments, it seems that with true underlying linear dynamics
(i.e.\ in the synthetic test cases), 
good results with rapid convergence may be obtained using a wide range of parameters.
On real datasets, tuning these took more effort, although here we limited our tuning
to the reuglarization parameters $\eta$ and $\beta$.

A second limitation is the potential sensitivity of TVART to initialization.
Due to non-convexity, there is no guarantee that Algorithm~\ref{alg:alt-min} converges
to a global optimum. 
In our experiments, we only found problems with this in the SST and ECoG datasets,
and adding noise to the initialization helped find better solutions.
We have verified that our results hold qualitatively across multiple random initializations,
as is standard with other tensor decompositions \citep{kolda2009}.
This is our recommended approach.

Finally, we have noted that the convergence sometimes stagnates for large problems,
in particular the SST and ECoG datasets.
This seems to affect the right spatial modes $\U2$ more severely than the others.
This may be due to the use of matrix-valued CG for that subproblem,
but it also seems likely that the subproblem for the right spatial modes 
is inherently more ill-conditioned with correlated data.
Multiresolution methods \citep{park2020} or explicit regularization of spatial modes
could improve convergence and produce more interpretable results
with large and spatially-correlated datasets.

Future work with TVART should investigate these and other possible extensions.
Higher-order autoregressive models \citep{west1997},
which incorporate multiple time delays stacked into Hankel matrices \citep{takens1981,gibson1992},
are a natural next step which would allow 
for wider applications, e.g.\ to economic data.
Theoretically, it would be good to have more understanding of what linear models
of nonlinear dynamics precisely find. 
It has recently been proven that some other non-convex factorization problems
have no spurious local minima \citep[e.g.][]{ge2017}.
Thus, it would be interesting to know whether there may exist similar results for
properly constrained CP decomposition models.
This might lead to modified versions of TVART with guaranteed optimality.

%{ \small
\section*{Acknowledgements}

%Anonymized.
Thank you to Scott Linderman and Nathan Kutz for discussions
and to Steven Peterson for help preprocessing the NeuroTycho data.
KDH was supported by a Washington Research Foundation Postdoctoral Fellowship.
BWB, RR, and KDH were supported by NSF CNS award 1630178 and
DARPA award FA8750-18-2-025.
%}

{\small
\setlength{\bibsep}{0pt plus 0.3ex}
\bibliographystyle{siamplain}
\bibliography{library}
}

\appendix

% \renewcommand{\theequation}{\thesection.\arabic{equation}}
% \renewcommand{\thealgorithm}{\thesection.\arabic{algorithm}}
% \setcounter{equation}{0}
% \setcounter{algorithm}{0}

% \begin{center}
%   { \Large
%   \bf{Supplemental Appendix}}

% {\large ``Time-varying Autoregression with Low Rank Tensors''}
% \end{center}

\section{Code availability}
MATLAB code implementing TVART and instructions for running it are available
from 
%\url{anonymized_url}.
\url{https://github.com/kharris/tvart}.

\section{Example application details}

\label{app:examples}

A complete table of parameters used for the test problems and datasets is given in Table~\ref{tab:datasets}.

\begin{table}[h!]
  \caption{TVART parameters}
  \label{tab:datasets}
  \centering
  \begin{tabular}{lllllllll}
    % \toprule
    % \multicolumn{2}{c}{Part}                   \\
    % \cmidrule(r){1-2}
    % Name     & Description     & Size ($\mu$m) \\
    % \midrule
    % Dendrite & Input terminal  & $\sim$100     \\
    % Axon     & Output terminal & $\sim$10      \\
    % Soma     & Cell body       & up to $10^6$  \\
    % \bottomrule
    \toprule
    Problem name & $N$ & $M$ & $T$ & $R$ &$\eta$ & $\beta$ & $\mathcal{R}$ & affine? \\
    \midrule
    Switching linear & 6--4000 & 20 & 10 & 4, 6 & $1/N$ & 5 & TV & no \\
    Smoothly varying linear & 6--4000 & 1 & 160 & 4 & $6/N$ & 600 $\log^2_{10}(N)$ & Spline & no \\
%    Lorenz system & 3 & 10 & 419 & 4 & 0.01 & 400 & TV & yes \\
    \href{https://github.com/AntonioCCosta/local-linear-segmentation}{Worm behavior} & 4 & 6 & 33 & 6 & 0.05 & 6 & TV & yes\\
    \href{https://www.esrl.noaa.gov/psd/repository/entry/show?entryid=12159560-ab82-48a1-b3e4-88ace20475cd}{Sea surface temperature} & 1259 & 9 & 171 & 6 & $10^{-3}$ & $10^4$ & Spline & no\\
    \href{http://neurotycho.org/food-tracking-task}{Neural activity} & 64 & 200 & 1999 & 8 & 1 & 100 & TV & no\\
    \bottomrule
  \end{tabular}
\end{table}

\end{document}